\documentclass[letterpaper, 10 pt, conference]{ieeeconf}  

\IEEEoverridecommandlockouts      

\usepackage{graphicx} 
\graphicspath{{Figures/}{./}} 
\usepackage{amsmath} 
\usepackage{amssymb}  

\usepackage[noadjust]{cite}

\usepackage{changebar}
\usepackage{color,soul}

\usepackage{mathtools}
\usepackage{tikz}
\usetikzlibrary{calc,3d}

\DeclareMathOperator{\Forall}{\forall}
\newtheorem{theorem}{Theorem}

\newtheorem{remark}{Remark}
\newcommand{\norm}[1]{\left\lVert#1\right\rVert}

\usepackage{siunitx}

\title{\LARGE \bf
Trajectory Tracking for Quadrotors with Attitude Control on $\mathcal{S}^2 \times \mathcal{S}^1$
}

\author{Dave Kooijman, Angela P. Schoellig, and Duarte J. Antunes
\thanks{Dave Kooijman and Duarte Antunes are with the Control Systems Technology group, Mechanical Engineering, Eindhoven University of Technology, The Netherlands. Emails: { dave.kooijman@robotics.utias.utoronto.ca}, {d.antunes@tue.nl}}
\thanks{Angela P. Schoellig is with the Dynamic Systems Lab (www.dynsyslab.org) at the University of Toronto Institute for Aerospace Studies (UTIAS), Canada. 	Email: { schoellig@utias.utoronto.ca}}%
\thanks{This research has been partially funded by the European Regional Development Fund (ERDF) as part of OPZuid 2014-2020 under the Drone Safety Cluster project.}
}

\begin{document}
\maketitle
\thispagestyle{empty}
\pagestyle{empty}

\begin{abstract}
The control of a quadrotor is typically split into two subsequent problems: finding desired accelerations to control its position, and controlling its attitude and the total thrust to track these accelerations and to track a yaw angle reference. While the thrust vector, generating accelerations, and the angle of rotation about the thrust vector, determining the yaw angle, can be controlled independently, most attitude control strategies in the literature, relying on representations in terms of quaternions, rotation matrices or Euler angles, result in an unnecessary coupling between the control of the thrust vector and of the angle about this vector. This leads, for instance, to undesired position tracking errors due to yaw tracking errors. In this paper we propose to tackle the attitude control problem using an attitude representation in the Cartesian product of the 2-sphere and the 1-sphere, denoted by $\mathcal{S}^2\times \mathcal{S}^1$. We propose a non-linear tracking control law on $\mathcal{S}^2\times \mathcal{S}^1$ that decouples the control of the thrust vector and of the angle of rotation about the thrust vector, and guarantees almost global asymptotic stability. Simulation results highlight the advantages of the proposed approach over previous approaches.
\end{abstract}

\section{Introduction}
\label{sec:introduction}
\par Quadrotors have been the focus of much research in the past decade~\cite{kumar:12,Schoellig:2012,hoffmann:11,Cai2014}. Typically, the controller has a cascade architecture consisting of an inner- and an outer-loop controller for attitude and position control, respectively, see, e.g.,~\cite{Mahony2012} and \cite{Invernizzi2018}. The outer-loop controller is designed for a simple double integrator model, providing virtual accelerations to control the quadrotor's positions. The actual accelerations result from the inner-loop controller, which controls the attitude of the quadrotor and the total thrust 
in order to track these accelerations and to track a yaw angle reference. An overview of this approach is given in Figure~\ref{fig:overview}.
\par The strategies to control the attitude of the quadrotor, see e.g.~\cite{Tayebi2006, Carino2015, Chonancova2016,Lee2010, Lefeber2017,casau:15}, i.e., to design the inner loop, often resort to general methods to control the attitude of a rigid body~\cite{wen:91,joshi:95,Mayhew2011,Fragopoulos2004}. 
For instance, the quaternion-based attitude control strategies in~\cite{wen:91} and~\cite{joshi:95} are applied to quadrotors in~\cite{Tayebi2006, Carino2015, Chonancova2016}; control strategies on the special orthogonal group, described, e.g., in~\cite{sanyal:08}, are applied in~\cite{Lee2010} and~\cite{Lefeber2017}; and a quaternion-based hybrid control law for attitude tracking in~\cite{Mayhew2011} is applied in~\cite{casau:15}. 
Note that, as stated in~\cite{Bath2000}, any continuous state-feedback control law, using local coordinates, is not globally well defined. This leads to unwinding, where the controller unnecessarily rotates the attitude through large angles, instead of global asymptotic stability. Therefore, it is impossible to stabilize any equilibrium point in the 3D rotations manifold, motivating discontinuous control laws (see, e.g., \cite{Lee2010}) and hybrid control laws (see, e.g.,~\cite{Mayhew2011}).

\begin{figure}[t]
	\centering
	\includegraphics[width=1.0\linewidth]{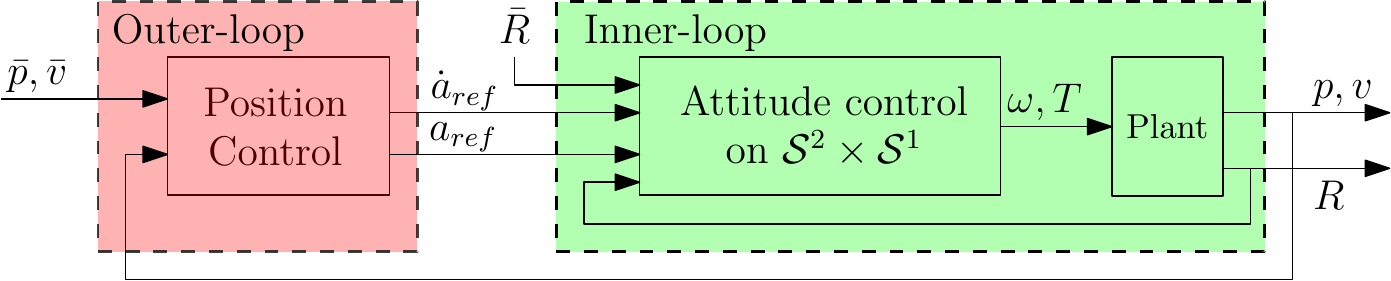}
	\caption{Control scheme with an outer-loop position controller generating reference accelerations, and an inner-loop attitude controller on $\mathcal{S}^2 \times \mathcal{S}^1$; $p,v,R$, represent position, velocity and attitude, and $\bar{p},\bar{v},\bar{R}$ corresponding references; $a_{ref}$ is a desired acceleration reference, $\omega$ is the angular velocity and $T$ is the total thrust. A yaw reference $\bar{\psi}$, along with $a_{ref}$, can determine $\bar{R}$ , see Section~\ref{sec:2_2}.}
	\label{fig:overview}
\end{figure}

\par However, these general attitude control strategies are not necessarily the most suitable to address the inner-loop quadrotor control problem. In fact, while the thrust vector of the quadrotor can be controlled independently of the angle of rotation about the thrust vector, which defines the yaw angle, these attitude control laws consider controlling the full attitude. This results in an unnecessary coupling between the control of the thrust vector and of the rotation about this vector. This is undesirable since a yaw error can result in a position error, and this is particularly noticeable for ``slow'' quadrotors where the ratio between the maximum torque that can be applied about the thrust vector and the moment of inertia about this axis is small.
\par In this paper, we propose to tackle the attitude tracking problem with a different attitude parameterization intended to decouple the control of the thrust vector from the control of the angle about this vector. To this end, we provide a convenient homeomorphism between $$\mathcal{SO}(3)\backslash \mathcal{R}_c, \ \ \mathcal{R}_c := \left\{ \begin{bmatrix}
-\alpha & \beta & 0 \\ \beta & \alpha & 0 \\ 0 & 0 & -1
\end{bmatrix} \mid \\ \left[\alpha \ \beta\right]^\top \in \mathcal{S}^1\right\}, $$
 and 
$$\left( \mathcal{S}^2 \backslash \{-e_3\}\right)\times \mathcal{S}^1, \ \ { e}_3 := \begin{bmatrix}
0 & 0 & 1
\end{bmatrix}^\top,$$
where $\mathcal{SO}(3) := \left\{R \in \mathbb{R}^{3\times3}|R^\top R =\mkern-3mu RR^\top \mkern-3mu =I, \det R=1\right\}$ is the {\it special orthogonal group} of order three, $\mathcal{S}^2$ is the 2-sphere and $\mathcal{S}^1$ is the 1-sphere, formally $\mathcal{S}^{n}:=\{x\in \mathbb{R}^{n+1}|x^\top x=1\}$, for $n \in \{1,2\}$. 
Where $\mathcal{R}_c$ can be interpreted as the set of orientations considering all rotations around the vector $-e_3$, these orientations are excluded to be able to define a homeomorphism.
This homeomorphism can be intuitively explained as follows. Given a rotation matrix in the special orthogonal group, which is not in $\mathcal{R}_c$, take the third column $r_3$, which belongs to $\mathcal{S}^2$, and the representation of the first column in the vectors of the orthogonal space to ${r}_3$ defined by the parallel transport of ${e}_1:=\begin{bmatrix}
1 & 0 & 0
\end{bmatrix}^\top$ and ${ e}_2:=\begin{bmatrix}
0 & 1 & 0
\end{bmatrix}^\top$ along the geodesic on $\mathcal{S}^2$ between ${e}_3$ and ${ r}_3$ (see Figure~\ref{fig:r3_ff}). Note that, for $({ v},w)\in \mathcal{S}^2\times \mathcal{S}^1$, with ${v} \neq -{e}_3$, ${ v}$ corresponds to the normalized thrust vector, and is equal to $e_3$ when hovering, and $w$ determines the angle about the thrust vector. 

\begin{figure}[t]
	\centering
	\includegraphics[width=.9\linewidth]{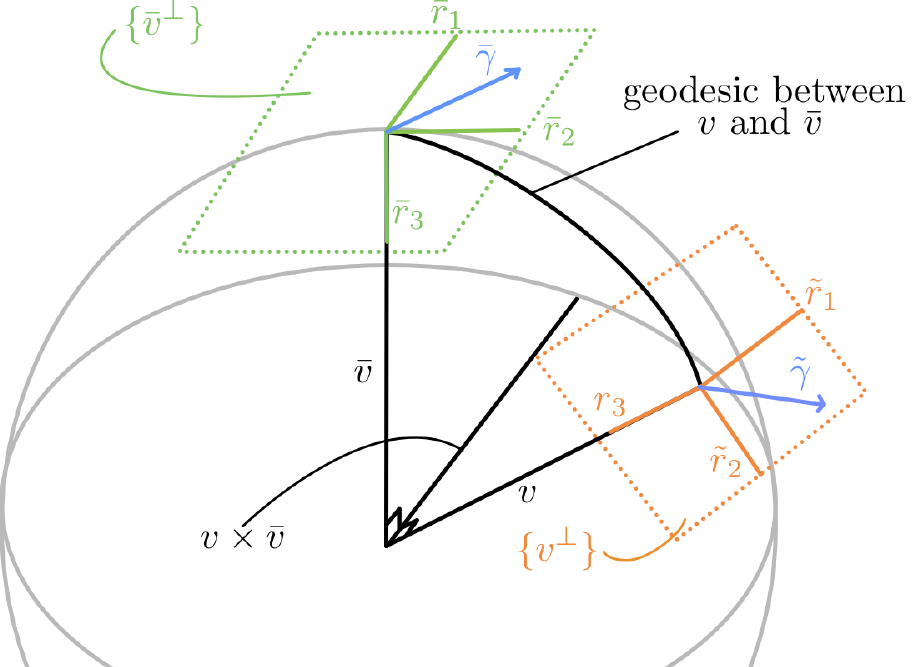} 
	\caption{Illustration of the parallel transport of vectors $\bar{r}_1$, $\bar{r}_2$ and $\bar{\gamma}$ from the tangent space to $\bar{v}=\bar{r}_3$ to the tangent space to $v=r_3 \neq -\bar{r}_3$ along the curve coinciding with the geodesic on $\mathcal{S}^2$ from $v=r_3$ to $\bar{v}=\bar{r}_3$. This results in vectors $\tilde{r}_1$, $\tilde{r}_2$ and $\tilde{\gamma}$ in the tangent space to $r_3$; $\bar{\gamma}$ is a general vector. Suppose that $\bar{v}=\bar{r}_3=e_3$ and $\bar{r}_1=e_1$, $\bar{r}_2=e_2$. If $\tilde{\gamma}=r_1$, its local coordinates in basis $\{\tilde{r}_1,\tilde{r}_2\}$ belong to $\mathcal{S}^1$ and together with $v$ fully characterize the rotation matrix from $\{r_1,r_2,r_2\}$ to $\{\bar{r}_1,\bar{r}_2,\bar{r}_3\}$ (homeomorphism between $\mathcal{S}\mathcal{O}(3)\backslash \mathcal{R}_c$ and $ \left( \mathcal{S}^2 \backslash \{-e_3\}\right)\times \mathcal{S}^1$). This figure is also useful to visualize the feedforward signal~\eqref{eq:ff}, which results from considering $\bar{\gamma}=\dot{\bar{r}}_3$ and $\tilde{\gamma} = \Theta \dot{\bar{r}}_3 $. }
	\label{fig:r3_ff}
\end{figure}

\par We then show that by applying a suitable input transformation, $v$ and $w$ can be controlled independently. We propose tracking control laws for the tracking subproblems in $\mathcal{S}^2$ and $\mathcal{S}^1$, considering the angular velocity as an input to the quadrotor. The proposed control law avoids coupling by construction. For differentiable attitude references, asymptotic convergence of the attitude tracking errors to zero as time converges to infinity is established, for every initial condition except at a single point, which has zero Lebesgue measure (following, e.g.,~\cite{monzon2003} we call this {\it almost global stability}).
\par In order to show the advantages of the new inner-loop control approach for quadrotor control, we introduce a simple trajectory tracking control law for the outer-loop generating acceleration references. 
The advantages of the proposed control law are then highlighted through simulations.
\par At the time of writing we found similar decoupling approaches for the attitude control of quadcopters in the literature (see \cite{Casau2015,Mueller2018,Brescianini2018,Markdahl2017}). Our approach differs from previous work in the parametrization we propose, which we believe provides a more natural way to define the rotation about the thrust vector.
\par This paper is organized as follows. Section~\ref{sec:problem_formulation} introduces the problem. The proposed approach is presented in Section~\ref{sec:inner_loop}. Section~\ref{sec:outer_loop} discusses an outer-loop trajectory tracking strategy to control the overall model. Section~\ref{sec:simulations} provides simulation results that show the advantages of the proposed solution compared to previous attitude control strategies for quadcopters, and conclusions are provided in Section~\ref{sec:conclusion}.

\section{Problem formulation}
\label{sec:problem_formulation}

\par Let $\{A\}$ denote a world fixed frame, and $\{B\}$ a body fixed frame centered at the vehicle's center of mass. The third axis of $\{A\}$, denoted by $e_3$, is assumed to be aligned with the gravity vector, as is the third axis of $\{B\}$, denoted by $r_3$,  when the quadrotor is hovering. Moreover, let ${ p} \in \mathbb{R}^3$ and ${ v} \in \mathbb{R}^3$ denote the position and velocity, respectively, of the center of mass of the quadrotor expressed in $\{A\}$.  Let $R = \left[ {\begin{array}{*{20}{c}} { r}_1 & { r}_2 & { r}_3 \end{array}} \right] \in \mathcal{SO}(3)$  denote the attitude of $\{B\}$ relative to $\{A\}$.
\par The forces acting on the quadrotor and expressed in the world fixed frame are assumed to be the gravity force, $mg{e}_3$, where $m\in \mathbb{R}$ and $g \in \mathbb{R}$ denote the mass and gravitational acceleration; the thrust force $-TR{ e}_3$, where $T \in \mathbb{R}_{> 0}$ denotes the total thrust generated by the blades. Therefore,
\begin{equation}\label{eq:quadrotor_model}
\begin{aligned}
\dot { p} &= { v}\\
m \dot { v} &= mg{ e}_3 - TR{ e}_3  \\
\dot R &= RS(\omega),
\end{aligned}
\end{equation}
where $\omega = \begin{bmatrix} \omega_1 & \omega_2 & \omega_3 \end{bmatrix}^\top \in \mathbb{R}^{3}$ is the angular velocity expressed in $\{B\}$, and for any $z \in \mathbb{R}^3$, we let
\begin{equation*}
{S(z)} := \left[ {\begin{array}{*{20}{c}}
	0&-z_3&z_2\\
	z_3&0&-z_1\\
	-z_2&z_1&0
	\end{array}} \right]\,.
\end{equation*}
\par The thrust vector, defined as $-TR{ e}_3=-T{ r}_3$, can then be controlled by controlling the last column of the rotation matrix. Often in the literature, and also here 
one assumes that the angular velocity can be directly controlled by controlling the motor speeds, i.e., $\omega$ is the control input of the model. Alternatively, one can consider that the thrust generated by each propeller generates torques, denoted by $\tau \in \mathbb{R}^3$, which are related to angular velocities by 
\begin{equation}\label{eq:quadrotor_model_d}
\mathcal{I}\dot\omega = -S(\omega) (\mathcal{I}\omega) + \tau,
\end{equation}
\noindent where
$\mathcal{I} \in \mathbb{R}^{3\times3}$ denotes the inertia matrix.


\subsection{Reference trajectory}\label{sec:2_2}
\par Consider a reference $(\bar{p}(t),\bar{v}(t),\bar{R}(t),\bar{\omega}(t))$ continuous in time $t
\in \mathbb{R}_{\geq 0}$ that satisfies~\eqref{eq:quadrotor_model} for a given initial condition, i.e., a reference that the quadrotor can track exactly and such that the corresponding total thrust satisfies
\begin{equation}\label{as:ref}
\bar{T} = \bar{r}_3^\top (-m \dot {\bar v} +mg{ e}_3) > 0,\ \forall t \in \mathbb{R}_{\ge 0}.
\end{equation} 
\par This defines a broad class of references, which we can consider with the tools provided in this paper, allowing acrobatic maneuvers where the angle between the gravity vector and $r_3$ is arbitrary. However, it will also be convenient to define the following subclass of references which we can easily parameterize and for which we can easily define a yaw angle relying on its orientation with respect to the gravity vector. These are
references for which the angle between the normalized reference thrust vector, i.e., $\bar{r}_3$, and the gravity vector does not exceed $\frac{\pi}{2}$, i.e.,
\begin{equation}\label{eq:r3restriction}
0 < e_3^\top \bar{r}_3 \le 1,
\end{equation}
for every $t\in \mathbb{R}_{\geq 0}$. We can parameterize this class with the position reference $\bar{ p}$ and an auxiliary angle $\bar \psi \in (-\pi,\pi]$, which plays the role of the yaw reference angle. The position reference $\bar{ p}(t)$, for which the first three derivatives are assumed to exist if~\eqref{eq:quadrotor_model} is considered or the first four derivatives are assumed to exist if~\eqref{eq:quadrotor_model},~\eqref{eq:quadrotor_model_d} are considered,  characterizes the third column of the rotation matrix $\bar{ r}_3$, and the thrust $\bar{T}$, since, from the second equation in \eqref{eq:quadrotor_model} we can write
\begin{equation}\label{eq:r3bar_}
\bar{r}_3 = \frac{b}{\|{ b}\|}, \ \ \bar{T}= m\|{ b}\|, \ \ { b}:=-\ddot{\bar{{ p}}}+ g{ e}_3.
\end{equation}
The reference angle $\bar \psi$, for which the first derivative is assumed to exist if~\eqref{eq:quadrotor_model} is considered or the first two derivatives are assumed to exist if~\eqref{eq:quadrotor_model},~\eqref{eq:quadrotor_model_d} are considered, is defined as the angle between ${e}_1$ and the projection of $\bar {r}_1$ on $({ e}_1, { e}_2)$. Therefore, it defines ${ \bar{r}}_1$ through
$$ \bar{r}_1  =  \frac{S(\bar r^\perp_{\psi}) \bar{r}_3}{s(\bar r^\perp_{\psi}, \bar{r}_3)},  \ \ \bar r^\perp_{\psi} := \begin{bmatrix} -\sin(\bar \psi) & \cos(\bar \psi) & 0 \end{bmatrix}^\top,$$
where, for $a,b \in \mathbb{R}^3$, $s({ a},{ b}):=\sqrt{1-({ a}^\top { b})^2}$.
Given ${ \bar{r}}_1$ and ${ \bar{r}}_3$, ${ \bar{r}}_2 = { S(\bar{r}}_3) { \bar{r}}_1$ and from $\dot{\bar{ R}}=\bar{ R}S(\bar{ \omega})$ we obtain
\begin{equation}\label{eq:r2omega}\bar{\omega}_{1}  = -\bar{r}_{2}^\top \dot{\bar{r}}_{3} , \ \ 
\bar{\omega}_{2}  =\bar{r}_{1}^\top \dot{\bar{r}}_{3}, \ \
\bar{\omega}_{3}  = \bar{r}_{2}^\top \dot{\bar{r}}_{1}.\end{equation}
\noindent If~\eqref{eq:quadrotor_model},~\eqref{eq:quadrotor_model_d} are considered, then
$
\bar{\tau} = \mathcal{I}\dot{\bar{\omega}} +S(\bar{\omega})(\mathcal{I}\bar{\omega}).
$

\subsection{Cascade trajectory tracking problem }\label{sec:2_3}
\par Following a common cascaded control design approach, we divide the tracking control problem into two subproblems: {\it(i)} find a virtual acceleration $a_{ref} \in \mathbb{R}^3$ to control the model
\begin{equation}\label{eq:f}
\dot p = v, \quad \dot v =a_{ref},
\end{equation}
such that the tracking errors $\tilde { p} = \bar { p} - { p}$ and $\tilde { v} = \bar { v} - { v}$ converge to zero, and {\it (ii)} find either $\omega$ and $T$ if the considered quadrotor model is~\eqref{eq:quadrotor_model} or $\tau$ and $T$ if \eqref{eq:quadrotor_model},~\eqref{eq:quadrotor_model_d}  are considered, such that the forces applied to the quadrotor (normalized by mass) track the desired virtual acceleration ${ a}_{ref}$, i.e., the acceleration error 
\begin{equation}
\tilde a := \left[ {\begin{array}{*{20}{c}} 0& 0& g \end{array}} \right]^\top + R \left[ {\begin{array}{*{20}{c}} 0& 0& \frac{-T}{m} \end{array}} \right]^\top - { a}_{ref},
\label{eq:SO3_error}
\end{equation} converges to zero, and $\bar R^\top R$ converges to the identity as time converges to infinity. We will denote the outer-loop control problem by {\it (i)}, and the inner-loop control problem by {\it (ii)}.
\par Note that driving~\eqref{eq:SO3_error} to zero constraints $r_3$ to coincide with the reference vector $\bar{r}_3$ and once such a condition is met, one can replace the goal $\bar R^\top R \rightarrow I$ by simply $ \bar r_1^\top r_1=1 \text{ and } \bar r_2^\top r_2=1,$
which is equivalent to stating that the angle of rotation about the normalized thrust vector (i.e., $r_3=\bar{r}_3$ when $\tilde a=0$) must coincide with the desired value. When~\eqref{eq:r3restriction} is met, this means that the yaw angle of the quadrotor coincides with the desired yaw reference angle $\bar \psi$.

\begin{remark}\label{rem:1}
	The inner-loop control problem for quadrotors is typically tackled directly in $\mathcal{SO}(3)$, or quaternion or Euler angles representations.  For instance, letting $\tilde{R}\in \mathcal{SO}(3)$ represent a reference attitude and $\tilde R:= \bar R^\top R $, the feedback controller (see \cite{Chaturvedi2011,Lefeber2017})
	\begin{equation}
	\omega(\tilde R) = K \sum_{i=1}^3  k_i { e}_i \times (\tilde R { e}_i),
	\label{eq:continuous_SO3_law}
	\end{equation}
	where $K = K^\top > 0 \in \mathbb{R}^{3\times 3}$, and  $k_i > 0 \in \mathbb{R}$ are distinct (e.g., $0 < k_1 < k_2 < k_3$)
	is such that $\tilde R(0)$ converges to $I$ for every initial condition $\tilde R(0) \in \mathcal{SO}(3) \backslash \{ \text{diag}(1,-1,-1),\text{diag}(-1,1,-1),\text{diag}(-1,-1,1)\}$. Moreover, a globally stabilizing discontinuous feedback quaternion-based law can be found in~\cite{Mayhew2011,Fragopoulos2004}. However, the position control of a quadrotor is only dependent on the third column of the rotation matrix (collinear with the thrust vector), as one attempts to drive the error~\eqref{eq:SO3_error} to zero. While the method we propose  drives the third column of $R$ to a desired value and controls independently  the angle about the third column of $R$ so that $R$ converges to a desired reference, both of these methods do not control these two quantities independently leading to undesired behaviors (see Section~\ref{sec:simulations}).
\end{remark} 
\section{Inner loop reference tracking in $\mathcal{S}^2\times \mathcal{S}^1$}
\label{sec:inner_loop}
\par We start by introducing, in Section~\ref{sec:3_1}, a convenient parametrization of $\mathcal{S}\mathcal{O}(3)\backslash \mathcal{R}_c$, by defining a homeomorphism between $\mathcal{S}\mathcal{O}(3)\backslash \mathcal{R}_c$  and $\left( \mathcal{S}^2 \backslash \{-e_3\}\right)\times \mathcal{S}^1$. In Section~\ref{sec:3_2} we define an input transformation that allows us to decouple the control of the normalized thrust vector on $\mathcal{S}^2\backslash \{-e_3\}$ from the control of the angle of rotation about the thrust vector, specified by a vector on $\mathcal{S}^1$. In Section~\ref{sec:3_3}, we consider the normalized thrust vector control problem in  $\mathcal{S}^2$, and in Section~\ref{sec:3_4}, we consider the control of the angle about the thrust vector in $\mathcal{S}^1$. Section~\ref{sec:3_5} combines the two strategies to provide the overall control law.

\subsection{Parameterization of $\mathcal{SO}(3)\backslash \mathcal{R}_c$}\label{sec:3_1}
\par We can parameterize $\mathcal{S}\mathcal{O}(3)\backslash \mathcal{R}_c$ as follows. Let $v \in \mathcal{S}^2\backslash \{-e_3\}$ and $w \in \mathcal{S}^1$. For $v\neq { e_3}$ define the following two perpendicular vectors in the tangent space to $ \mathcal{S}^2$ at $v$, 
$$ a_1^\perp(v):=-\frac{(I-vv^\top)}{s({ e_3},v)}{ e_3},  \ \ a_2^\perp(v):=\frac{S({e_3})}{s( {e_3}, v)}v,$$
where $I$ denotes the $3 \times 3$ identity matrix.
Note that $a_2^\perp(v)$ is also in the tangent space to $\mathcal{S}^2$ at ${ e_3}  $. 
Let  
$$ b_1^\perp(v):=-\frac{(I- {e_3} {e_3}^\top)}{s({ e_3},v)} v$$
be an extra vector in the tangent space to $\mathcal{S}^2$ at $e_3$, which is perpendicular to $a_2^\perp(v)$. The first component of these tangent vectors at ${ e_3} $ are denoted by
$$ d_1(v)= e_1^\top b_1^\perp(v), \ \ d_2(v)= e_1^\top a_2^\perp(v).$$
Moreover, define the following rotation matrices
$$ 
\begin{aligned}
U(v)&:= \begin{bmatrix}
a_1^\perp(v) & a_2^\perp(v) & v
\end{bmatrix}\\
V(v)&:= \begin{bmatrix}d_1(v) & -d_2(v) & 0 \\ d_2(v) & d_1(v) & 0  \ \\ 0 & 0 & 1\end{bmatrix}
\\
W(w)&:= \begin{bmatrix}w_1 & -w_2 & 0 \\ w_2 & w_1 & 0  \ \\ 0 & 0 & 1\end{bmatrix}\,,
\end{aligned}
$$
for $w=(w_1,w_2)$. Then, the following map is the desired homeomorphism
$\phi: \left( \mathcal{S}^2 \backslash \{-e_3\}\right)\times \mathcal{S}^1 \mapsto \mathcal{S}\mathcal{O}(3)\backslash \mathcal{R}_c 
$ 
$$ \phi(v,w)= \left\{\begin{aligned}
&U(v)V(v)W(w) \text{ if }e_3 \neq v \\
&W(w) \text{ if }e_3=v, \end{aligned}\right.$$
where $w= (w_1,w_2)$. The inverse $
\phi^{-1}: \mathcal{S}\mathcal{O}(3)\backslash \mathcal{R}_c  \mapsto
\left( \mathcal{S}^2 \backslash \{-e_3\}\right)\times \mathcal{S}^1
$ is given by 
$$ \phi^{-1}\!(R) \!=\! \left\{\begin{aligned}
&(r_3, \!\begin{bmatrix}
d_1(r_3) & d_2(r_3) \\ 
-d_2(r_3) & d_1(r_3)
\end{bmatrix}  \begin{bmatrix}
r_1^\top a_1^\perp(r_3) \\ 
r_1^\top a_2^\perp(r_3) 
\end{bmatrix}).\!\text{ if }\!r_3\!\neq\! e_3,\\
& (e_3,\begin{bmatrix}
e_1^\intercal r_1 \\
e_2^\intercal r_1
\end{bmatrix})\text{ if }r_3=e_3.\\
\end{aligned}\right.$$

In fact, it is clear that $\phi$ and $\phi^{-1}$ are one-to-one and continuous; therefore, $\phi$ defines an homeomorphism between $
\left( \mathcal{S}^2 \backslash \{-e_3\}\right)\times \mathcal{S}^1$ and $ \mathcal{S}\mathcal{O}(3)\backslash \mathcal{R}_c $.

\subsection{Input transformation}\label{sec:3_2}
\par In this section, we focus on the control of the attitude described by the equation
\begin{equation}\label{eq:3}
	 \dot{ R}=RS(\omega),
\end{equation}
and consider the standard stabilization problem of finding a control law for $\omega$ that drives $R$ to the identity matrix $\bar{R}=\begin{bmatrix}
\bar{r}_1 & \bar{r}_2 & \bar{r}_3
\end{bmatrix}=I$ for a given initial condition $R(0)=R_0 \in \mathcal{S}\mathcal{O}(3)$. The tracking problem can either be converted into a stabilization problem by defining $\tilde{R} = \bar{R}^\top R$ and noticing that $\dot{\tilde{R}}=\tilde{R}S(\tilde{\omega})$, for $\tilde{\omega}=\omega-R^\top \bar{R}\bar{\omega}$ or tackled as we do in the sequel. 
\par For this problem, we define an input transformation 
$ \tilde{\omega} = g(R,\omega),$
where $\tilde{\omega}=(\tilde{\omega}_1,\tilde{\omega}_2,\tilde{\omega}_3)$, such that, for $(v,w)=\phi^{-1}(R)$,
\begin{align} \label{eq:v1}
\dot{v} &=f_A(v,\tilde{\omega}_1,\tilde{\omega}_2 )\\
\dot{w} &=f_B(w,\tilde{\omega}_3). 
\label{eq:w1}
\end{align}
Note that with this transformation we can control independently $v \in \mathcal{S}^2$ and $w\in \mathcal{S}^1$.
\par To this effect, we start by noticing that if we make
$$\omega_1 = -r_2^\top u_v, \quad \omega_2 = r_1^\top u_v, $$
for some control input $u_v \in \mathbb{R}^3$, and replace these expressions in the  third column of~\eqref{eq:3}, i.e., $\dot{r}_3=\omega_2r_1-\omega_1r_2$, we obtain
\begin{equation*}
		\dot { r}_3 = { r}_1{ r}_1^\top { u}_v + { r}_2{ r}_2^\top { u}_v = (I - { r}_3{ r}_3^\top ){ u_v},
\end{equation*}
or, equivalently, substituting $r_3 = v$, 
\begin{equation}
	\dot v  = (I - vv^\top ){ u_v}.
		\label{eq:r3dot2}
\end{equation}
Note that the component of $u_v$ parallel to $v$ plays no role in \eqref{eq:r3dot2}. Therefore, considering that ${u_v}$ belongs to the tangent space to $v$, that is,
\begin{equation}\label{eq:u_}
	{u_v} = \tilde{\omega}_1 z_1+\tilde{\omega}_2 z_2,
\end{equation}
for some vectors $z_1$ and $z_2$ in the tangent space of $\mathcal{S}^2$ at $v$, 
we conclude that \eqref{eq:r3dot2}~and~\eqref{eq:u_} take the form~\eqref{eq:v1}.
\par Let $\bar{v}=\bar{r}_3$ and define
\begin{equation}\label{eq:7}
	\begin{aligned}
		R_e &:= \begin{bmatrix} \frac{S(\bar v)}{s(\bar v,v)} v & -\frac{(I-{ \bar v \bar v^\top)}}{s(\bar v,v)}v & \bar v \end{bmatrix}\\
		R_r &:= \begin{bmatrix} \frac{S(\bar v)}{s(\bar v,v)}v & \frac{(I-vv^\top)}{s(\bar v,v)}\bar v & v \end{bmatrix} \,.
	\end{aligned}
\end{equation}
The parallel transport from the space orthogonal to $\bar v$ to the space orthogonal to $v$ along the geodesic curve in $\mathcal{S}^2$ connecting the two is given by
$
	R_rR_e^\top z,
$
for a vector $z$ that belongs to the tangent space to $\mathcal{S}^2$ at $ \bar v$.
We can then use the parallel transport for ${ \bar{ r}}_1$ and ${ \bar{ r}}_2$ and obtain a local frame in the tangent space of $v$,
$
	\left\{ R_rR_e^\top { \bar{ r}}_1, R_r R_e^\top { \bar{ r}}_2 \right\},
$
the angle between this frame and $\{ { r}_1,{ r}_2\}$ is the error angle. 
Note that $w = (w_1,w_2)$, where
\begin{equation}\label{eq:18}
\begin{aligned}
w_1 &= r_1^\top R_r R_e^\top { \bar{ r}}_1 \\
w_2 &= r_2^\top R_r R_e^\top { \bar{ r}}_1,
\end{aligned}
\end{equation}
Moreover, $ w_1^2+w_2^2=1$,
and the error angle is $\text{atan2}(w_2,w_1)$. 
We wish to drive this angle to zero, or equivalently, drive $w_1$ to  $1$ and  $w_2$ to $ 0$.
Let $\omega_e$ and $\omega_r$ be obtained in a similar fashion to~\eqref{eq:r2omega} and be such that
$$\dot{R}_e = R_eS(\omega_e), \quad \dot{R}_r= R_rS(\omega_r).$$

\par Using these expressions, and the fact that, from~\eqref{eq:3}, $\dot{r}_1= \omega_3 r_2-\omega_2 r_3$ and $\dot{r}_2= -\omega_3 r_1+\omega_1 r_3$, we can now write
\begin{equation}\label{eq:syspsi}
	\begin{aligned}
		\dot{w}_1 &= \omega_3 w_2- \underbrace{(\omega_2 r_3^\top R_r R_e^\top { \bar{ r}}_1-r_1^\top R_r S(\omega_r-\omega_e)R_e^\top { \bar{ r}}_1)}_{:=\beta_1} \\
		\dot{w}_2 &=  -\omega_3 w_1+ \underbrace{\left(\omega_1 r_3^\top R_r R_e^\top { \bar{ r}}_1+r_2^\top R_r S(\omega_r-\omega_e)R_e^\top { \bar{ r}}_1\right)}_{:=\beta_2}.\\
	\end{aligned}
\end{equation}
Since $w_1^2+w_2^2=1$ we have $w_1\dot{w}_1+w_2\dot{w}_2=0$,
from which we conclude that $ w_2\beta_2-w_1\beta_1=0$,
or, assuming that $w_1 \neq 0$ and $w_2 \neq 0$, $ \frac{\beta_2}{w_1}=\frac{\beta_1}{w_2}$.
Since we might have $w_1 = 0$ or $w_2= 0$ but not simultaneously, the following is bounded
$$ \omega_{r} = \left\{\begin{aligned}
&\frac{\beta_2}{w_1}\text{ if }|w_1|>|w_2|\\
&\frac{\beta_1}{w_2}\text{ otherwise.}
\end{aligned}\right.$$
Then we can write
\begin{equation}\label{eq:1}
\dot{w}_1 = (\underbrace{\omega_3 -\omega_r}_{:=u_w})w_2, \quad
\dot{w}_2 =  -(\underbrace{\omega_3-\omega_r}_{:=u_w}) w_1,
\end{equation}
which, by considering $\tilde{\omega}_3=u_w$, takes the form~\eqref{eq:w1}.

\subsection{Controlling $v \in \mathcal{S}^2$}\label{sec:3_3}
\par Let $\bar{ v}={ \bar{r}}_3$ be the reference for $v$. 
We propose the following control law to control~\eqref{eq:r3dot2},
 \begin{equation}\label{eq:input}
 { u_v} = { u}_{v,\text{FB}} + { u}_{v,\text{FF}} \in \mathbb{R}^3,
 \end{equation}
 where ${ u}_{v,\text{FB}}$ and  ${ u}_{v,\text{FF}}$ can be seen as feedback and feedforward laws as detailed next. 
 \par For feedback, we consider
\begin{equation}
{ u}_{v,\text{FB}} = \begin{cases}
k_1\bar{v} & \text{if $v^\top \bar{ v} \ge 0$} \\
k_1 { r}_1 & \text{if $v = -\bar{ v}$}\\
k_1\frac{1}{s(v,\bar{v} )}\bar{v} & \text{otherwise},
\end{cases}
\label{eq:input_u2}
\end{equation}
with $k_1 > 0$. Intuitively, this control law is such that when used in~\eqref{eq:r3dot2}, $(I-vv^\intercal){ u}_{\text{FB}}$ becomes a vector of the tangent space at $v$ proportional to the derivative at $v$ of the geodesic curve from $v$ to $\bar{v}$ (the curve connecting $v$ and $\bar{v}$ in $\mathcal{S}^2$ with shortest distance). The norm of this tangent vector is unitary in $\{v \in \mathcal{S}^2|v^\top \bar{v}\leq 0\}$ and decreases as $v$ approaches $\bar{v}$ in $\{v \in \mathcal{S}^2|v^\top \bar{v}>0\}$. Note that when $v = -\bar{v}$ the geodesic is not unique. Therefore the control law corresponds to a rotation in the direction of ${r}_1$, while in fact, any unitary vector in the plane $({r}_1,{r}_2)$ can be used.
\par For the feedforward term, we consider
\begin{equation}\label{eq:ff}
{ u}_{v,\text{FF}} = \begin{cases}
\dot{\bar{v}} & \text{if $v = \bar{v}$}\\
-\dot{\bar{v}} & \text{if $v = -\bar{v}$}\\
\Theta \dot{\bar{v}}& \text{otherwise,}
\end{cases}
\end{equation}
with $\Theta =  \frac{1}{s(v, \bar{ v})^2}((v \times \bar{ v})(v \times \bar{ v})^\top -  (I - vv^\top )\bar{v}v^\top)$, 
which maps $\dot{\bar{ v}}$ to the tangent space of $\mathcal{S}^2$ at $v$ (see Fig.~\ref{fig:r3_ff}). The intuition behind this feedforward control law is that the angle between $v$ and $\bar{v}$ remains constant when only the feedforward input is applied.


Now we will look at the stability of this control law, which is summarized in the next result.


\begin{theorem}
Let $\bar{ v}(t)$ for $t \in \mathbb{R}_{\geq 0}$ define a given differentiable curve in $\mathcal{S}^2$  and consider \eqref{eq:r3dot2} with control law~\eqref{eq:input},~\eqref{eq:input_u2} and~\eqref{eq:ff}. Then $\lim_{t \rightarrow \infty}v = \bar{v}$ for every initial condition $v(0)\in\mathcal{S}^2$.
\label{theorem:Stab_v}
\end{theorem}

\begin{proof} Proving that $v$ converges to $\bar{v}$ as $t$ converges to infinity is equivalent to proving that $\xi:=\bar{v}^\intercal v$ converges to $1$ since
$\frac{1}{2}(\bar{v}- v)^\intercal (\bar{v}- v) = 1-\bar{v}^\top v$. We shall prove that for every $v(0)\neq \bar{v}(0)$,
\begin{equation}\label{eq:F3}
\dot{\xi} = \left\{\begin{aligned} &k_1(1-\xi^2), \text{ if }0\leq \xi \leq 1, \\
 &k_1\sqrt{(1-\xi^2)}\text{ if }-1< \xi<0.
\end{aligned}\right.
\end{equation}
Then, we can compute a closed-form expression for $\xi$: if $0\leq \xi(0) \leq 1$,
$$ \xi(t)=\frac{\kappa e^{2k_1t}-1}{\kappa e^{2k_1t}+1}, \text{ for every }t\geq 0,$$
where $\kappa:=\frac{1+\xi(0)}{1-\xi(0)}$ and if $-1<\xi(0) < 0$, letting $c_1 := \text{asin}(\xi(0))$ and $t_s:=-\frac{c_1}{k_1}$
$$ \xi(t) = \left\{\begin{aligned}
&\frac{e^{2k_1t}-1}{e^{2k_1t}+1}, \text{ if }t\geq t_s\\
&\sin(k_1t+c_1),\text{if } 0\leq t < t_s
\end{aligned}\right.,$$
By taking the limit as $t$ converges to infinity, we conclude that $\xi$ converges to 1.
\par It rests to prove~\eqref{eq:F3}. The time derivative of $\xi$ is
\begin{equation*}
\begin{aligned}
\dot \xi&= (\dot{\bar{v}}^\top v + \bar{v}^\top \dot{v})= (\dot{\bar{v}}^\top v + \bar{v}^\top (I-{v}v^\top)  ({ u}_{v,\text{FB}} + { u}_{v,\text{FF}})).\\
\end{aligned}
\end{equation*}
Note that  $\bar{ v}^\top (I - vv^\top){u}_{v,\text{FF}} = -\dot{\bar{v}}^\top v,$ 
and therefore
$\dot \xi = \bar{v}^\top (I - vv^\top){ u}_{v,\text{FB}},$
which can be rewritten as
\begin{equation}\label{eq:FF}
\dot \xi = \begin{cases}
\bar{v}^\top (I-v v^\top)k_1 \bar{v}& \text{if $v^\top \bar{v} \ge 0$}\\
\bar{v}^\top (I-v v^\top)k_1 { r}_{1} = 0 & \text{if $v = - \bar{v}$}\\
\bar{v}^\top (I-v v^\top)\frac{k_1 \bar{v}}{s(v, \bar{ v})} & \text{otherwise}.
\end{cases}
\end{equation}
Since $v(0) \neq - \bar{v}(0)$, then $\xi(0)\neq -1$, and since $\xi(t)$ monotonically increases, $v(t)\neq -\bar{v}(t)$, for every $t \geq 0$. Then~\eqref{eq:FF} implies~\eqref{eq:F3}, concluding the proof.

\end{proof}

\subsection{Controlling $w \in S^1$}\label{sec:3_4}
\par In this section, we consider an arbitrary function of time $\bar{R}(t)$ (not necessarily the identity). It is convenient therefore to define $(.,w)=\phi^{-1}(\tilde{R})$, where $\tilde{R} = \bar{R}^\top R$, which is still described by~\eqref{eq:18} but now for time-varying $\bar{R}$. Following the same steps as in Section~\ref{sec:3_2}, we arrive at~\eqref{eq:syspsi} with $\beta_1$ and $\beta_2$ replaced by $ \bar{\beta}_i = \beta_i+\theta_i, \quad i \in \{1,2\},$ where 
$\theta_1:=-r_1^\top R_r R_e^\top { \dot{\bar{ r}}}_1$, $\theta_2:=r_2^\top R_r R_e^\top { \dot{\bar{ r}}}_1$. As a result, instead of~\eqref{eq:1}, we obtain
\begin{equation*}
\dot{w}_1 = (u_w-\zeta)w_2, \quad
\dot{w}_2 =  -(u_w-\zeta) w_1,
\end{equation*}
where
\begin{equation*}
\zeta := \left\{\begin{aligned}
& \frac{\theta_2}{w_1}\text{  if }|w_1|> |w_2| \\
& \frac{\theta_1}{w_2}\text{  otherwise. }
\end{aligned}\right.
\end{equation*}
\par We propose the following control law
\begin{equation}\label{eq:5}
	{ u}_w = { u}_{w,\text{FB}} + { u}_{w,\text{FF}} \in \mathbb{R},
\end{equation}
where ${ u}_{w,\text{FB}}$ and  ${ u}_{w,\text{FF}}$ can be seen as feedback and feedforward laws. For feedforward, we consider $u_{w,\text{FF}}= \zeta $ and for feedback we consider
 \begin{equation}\label{eq:2}
  u_{w,FB}= \left\{\begin{aligned}
 & k_2w_2 \text{ if }w_1 \geq 0 \\
 & -k_2 \text{ if }w_1 <0\text{ and }w_2<0\\
 &k_2\text{ if }w_1 <0\text{ and }w_2\geq 0,
 \end{aligned}\right.
 \end{equation}
 for some positive constant $k_2$. 
 \par The next result establishes that this control  drives $w_1$ to 1 and $w_2$ to $  0$, as desired.

\begin{theorem}
Let $\bar{R}(t)$ for $t \in \mathbb{R}_{\geq 0}$ define a given differentiable curve in $\mathcal{S}\mathcal{O}(3)$. Consider the system \eqref{eq:1} with control law~\eqref{eq:5},~\eqref{eq:2}. Then
\begin{equation}
\lim_{t \rightarrow \infty}w_1= 1\text{ and }\lim_{t \rightarrow \infty}w_2=0,
\end{equation}
for every initial condition $w(0)\in\mathcal{S}^1$.
\label{theorem:Stab_w}
\end{theorem}
\begin{proof}
	Replacing the control law $u_w$, we obtain
\begin{equation*}
		\dot{w}_1 = u_{w,\text{FB}}w_2, \quad
		\dot{w}_2 =  -u_{w,\text{FB}} w_1.
\end{equation*}
 Consider the Lyapunov function
\begin{equation*}
V(w_1,w_2) = (w_1 - 1)^2 + w_2^2 = 2(1-w_1),
\end{equation*}
which is always positive definite. The derivative is given by
$\dot V(w_1, w_2) = -\dot w_1$
and can be rewritten as
\begin{equation}
\dot V(w_1,w_2) = \begin{cases} 
-k_2w_2^2 & \text{if $w_1 \ge 0$} \\
k_2w_2 & \text{if $w_1 < 0$ and $w_2 < 0$}\\
-k_2w_2 & \text{if $w_1 < 0$ and $w_2 \ge 0$}.
\end{cases}
\end{equation}
Therefore, the Lyapunov function derivative is negative semi-definite, and in the set $\Omega:=\{(w_1,w_2)\in \mathbb{R}^2|V(w_1,w_2)=0\}$, i.e., when $w_2=0$, we have $\dot{w}_2=k_2$ if $w_1=-1$, $\dot{w}_2=0$ if $w_1=1$. Therefore, the only solution that can stay in $\Omega$ is $w_1 = 1$, $w_2=0$ and thus according to LaSalle's Theorem \cite{Khalil_book} we have that $(1,0)$ is globally asymptotically stable.
\end{proof}

\subsection{Inner-loop control law}\label{sec:3_5}
\par The combined control law in $\mathcal{S}^2 \times \mathcal{S}^1$ allows for~\eqref{eq:3} to track a given reference $\bar{R}$ as stated in the next result, which results from the arguments given in the previous sections.
\begin{theorem}
	Let $\bar{R}(t)$ for $t \in \mathbb{R}_{\geq 0}$ define a given differentiable curve in $\mathcal{S}\mathcal{O}(3)$  and consider \eqref{eq:3} with control law
	\begin{equation}
	\omega = \begin{bmatrix} -r_2^\top u_v & r_1^\top u_v &u_w+\omega_r \end{bmatrix}^\top,
	\end{equation}
	where $u_v$ are $u_w$ are obtained by replacing $v$ by $r_3$ and $\bar{v}$ by $\bar{r}_3$ in~\eqref{eq:input},~\eqref{eq:input_u2} and~\eqref{eq:ff}, and in~\eqref{eq:7},~\eqref{eq:18},~\eqref{eq:syspsi} respectively. Moreover, let, $\tilde{R} := \bar{R}^\top R$. Then
	\begin{equation}
	\lim_{t \rightarrow \infty} \tilde{R} = I
	\end{equation}
	for every initial condition $R(0)$ such that $\tilde{ R}\in \mathcal{S}\mathcal{O}(3)\backslash \mathcal{R}_c$.
\end{theorem} 
\par Since asymptotic convergence of $R$ to $\bar R$ is obtained for every initial condition except in a set of measure zero, this result assures almost global convergence, following the standard terminology in the literature \cite{monzon2003}.

\section{Outer loop and overall control law}
\label{sec:outer_loop}
\par Following the control approach of Section~\ref{sec:2_3}, in this section we consider the position tracking problem of finding virtual acceleration input references ${ a}_{ref}$ for~\eqref{eq:f} that can track a certain position and velocity trajectory. The inner loop must make the error~\eqref{eq:SO3_error} zero. Note that when~\eqref{eq:SO3_error} is zero, then ${ a}_{ref}\in 	\mathcal{A}$ where
	$$ 
	\mathcal{A}:=\left\{ \left[ \ 0 \ \ 0 \ \ g \ \right]^\top -r_3 \frac{T}{m} \middle| 0<T<T_M, \|r_3\|= 1\right\}$$
	is a ball in $\mathbb{R}^3$ and $T_M$ is the maximum total thrust that can be applied. Since, for a vector $z\in \mathbb{R}^3$, $ \|z\|_{2} \leq \sqrt{3}\|z\|_{\infty}$, $\mathcal{A}_\infty \subseteq \mathcal{A}$ where
	$$
	\mathcal{A}_\infty\!:=\!\left\{ \left[ \ 0 \ \ 0 \ \ g \ \right]^\top \!\!-r_3 \frac{T}{m} \middle| 0<T<T_M, \|r_3\|_\infty= \frac{1}{\sqrt{3}}\right\}$$
	is a box in $\mathbb{R}^3$. It is more convenient to restrict $a_{ref}$ to belong to $\mathcal{A}_\infty$ because this imposes an independent restriction on each of the three components of  $a_{ref}$. We focus on the following single-input single-output (SISO) system
	 \begin{equation}
	 \dot x_1 =x_2,\quad
	 \dot x_2 = u,
	 \label{eq:double_integrator_model}
	 \end{equation}
	 where $x_1$ and $x_2$ represent the position and velocity, respectively, and the control law for $u$ must satisfy $u \in [\underline{L},\overline{L}]$ for given constants $\underline{L}$ and $\overline{L}$. We let $p=\begin{bmatrix} p_x & p_y & p_z\end{bmatrix}^\top$, where $p_x,p_y,p_z$ are the positions in the direction of $e_1, e_2,$ and $e_3$, respectively. Then, for the system pertaining to $p_x$ and $p_y$, $\underline{L}=-\frac{\bar{ T}}{m\sqrt{3}}$, $\overline{L}=\frac{\bar{ T}}{m\sqrt{3}}$ and for the system pertaining to $p_z$, $\underline{L}=g-\frac{\bar{ T}}{m\sqrt{3}}$, $\overline{L}=g+\frac{\bar{ T}}{m\sqrt{3}}$. If we further restrict $\overline{L}\leq g$ for the system pertaining to $p_z$, we assure that ~\eqref{eq:r3restriction} holds, which is convenient as explained at the end of Section~\ref{sec:2_3}, since then the angle about $\bar{r}_3$ can be specified by a predefined function $\bar{\psi}$, playing the role of a yaw angle reference. In fact, setting~\eqref{eq:SO3_error} to zero, multiplying by $e_3^\intercal$, noticing that $T$ must be non-negative and using~\eqref{eq:r3restriction}, we conclude that 
	$ e_3^\top a_{ref} \leq g,$ which simply states that the required acceleration downwards cannot be larger than the gravitational acceleration when~\eqref{eq:r3restriction}  is met.
	\par In addition to the saturations on the control input $a_{\text{ref}}$ for~\eqref{eq:f}, we must assume that it is differentiable, if the considered quadrotor model is~\eqref{eq:quadrotor_model} or twice differentiable if \eqref{eq:quadrotor_model},~\eqref{eq:quadrotor_model_d} are considered. Although we consider simply~\eqref{eq:quadrotor_model}, we provide a twice-differentiable control input to cope also with cases where \eqref{eq:quadrotor_model},~\eqref{eq:quadrotor_model_d} are considered.
	\par In Section~\ref{sec:outerloop}, we tackle the problem of finding a saturated twice-differentiable reference tracking control law for~\eqref{eq:double_integrator_model} and, in Section~\ref{sec:overall}, we provide the proposed overall control law.

\subsection{Outer loop}\label{sec:outerloop}
\par We are interested in tracking a reference trajectory $(\bar x_1, \bar x_2)$ and write the error system with states $e_1 = \bar x_1 - x_1$ and $e_2 = \bar x_2-x_2$ as
\begin{equation}
\dot{e}_1 = e_2, \quad \dot{e}_2 = \bar u - u.
\end{equation}

Note that since we have constraints on the input $u$, this also imposes constraints on the reference, namely,
\begin{equation}
\bar u = \dot {\bar x}_2 \in (\underline{L},\overline{L}).
\end{equation}
Consider the control law
\begin{equation}
u = \text{sat}^\epsilon_{(\underline{L},\overline{L})}(\bar u + k_1 e_1 + k_2 e_2),
\label{eq:double_integrator_control_law}
\end{equation}
where $k_1, k_2 \in \mathbb{R}_{>0}$ and
\begin{equation}
\text{sat}^\epsilon_{(a,b)}(x) := \begin{cases}
x & \text{if $x \in (a+\epsilon,b-\epsilon)$}\\
a & \text{if $x \in (\infty, a-\epsilon]$}\\
b & \text{if $x \in [b+\epsilon, \infty)$}\\
x+\frac{1}{4\epsilon}(x-(a+\epsilon))^2 & \text{if $x \in (a-\epsilon,a+\epsilon]$}\\
x-\frac{1}{4\epsilon}(x-(b-\epsilon))^2 & \text{if $x\in [b-\epsilon, b+\epsilon)$,}
\end{cases}
\end{equation}
with $\epsilon \in [0, \frac{b-a}{2}]$. Since $\underline{L} < \bar u < \overline{L}$, we can rewrite \eqref{eq:double_integrator_control_law} to
\begin{equation}
u =\bar u + \text{sat}^\epsilon_{(\underline{L}-\bar u,\overline{L}-\bar u)}(k_1 e_1 + k_2 e_2).
\end{equation}
The error dynamics for the closed-loop system become
\begin{equation}
\begin{aligned}
\dot e_1 &= e_2\\
\dot e_2 &= -\text{sat}^\epsilon_{(a,b)}(k_1 e_1 + k_2 e_2),
\end{aligned}
\end{equation}
where $a = \underline{L}-\bar u \in (\underline{L}-\overline{L},0)$ and $b = \overline{L} - \bar u \in (0,\overline{L}-\underline{L})$. 

\subsection{Overall control law}\label{sec:overall}
\par The control law for the outer-loop~\eqref{eq:double_integrator_control_law} provides a virtual acceleration reference vector, denoted by ${ a}_{ref}=\begin{bmatrix} u_x & u_y & u_z\end{bmatrix}^\top$,where $u_x$, $u_y$, $u_z$ are given by \eqref{eq:double_integrator_control_law} for the corresponding error variables corresponding to $p_x$, $p_y$ and $p_z$. In this section, we use $a_{ref}$ to find a control input $\omega$ and $T$ to reduce the error~\eqref{eq:SO3_error} considering the dynamics given in~\eqref{eq:quadrotor_model}.
We define $\bar { a}_{ref} := \left[ {\begin{array}{*{20}{c}} 0& 0& g \end{array}} \right]^\top - { a}_{ref}$. The error \eqref{eq:SO3_error} can be rewritten as 
\begin{equation}
\tilde a = \bar { a}_{ref} - { r}_3 m T = \bar{r}_3 \norm{\bar { a}_{ref}} -{ r}_3 m T,
\end{equation}
where $\bar{r}_3 = \frac{\bar { a}_{ref}}{\norm{\bar { a}_{ref}}}$ represents the direction, and $\norm{\bar { a}_{ref}}$ the magnitude of $\bar { a}_{ref}$. 
Note that $\norm{\bar {a}_{ref}} > 0$. A natural choice for the thrust is $ T = m\norm{\bar { a}_{ref}}$. Moreover, $\bar{r}_3$ provides the reference for the inner-loop. Note that $\bar{r}_3$ must be now computed online and depends on the position and velocity errors in~\eqref{eq:double_integrator_control_law}. By picking $\underline{L},\overline{L}$, the saturation constants in Section~\ref{sec:outerloop}, sufficiently small we can guarantee that~\eqref{eq:r3restriction} holds. Therefore, $\bar{r}_1$ and $\bar{r}_2$ are completely characterized by a desired yaw angle $\bar \psi$ as explained in Section~\ref{sec:2_2}.
\par The overall control law is then
\begin{equation}
\begin{aligned}
\bar { a}_{ref} &= \left[ {\begin{array}{*{20}{c}} 0& 0& g \end{array}} \right]^\top -\left[ {\begin{array}{*{20}{c}} u_x& u_y& u_z \end{array}} \right]^\top \\
T &= m\norm{\bar { a}_{ref}}, \quad  \bar{r}_3 = \frac{\bar { a}_{ref}}{\norm{\bar { a}_{ref}}}\\
\bar{r}_1 &= \frac{S(\bar r^\perp_{\psi}) \bar{r}_3}{s(\bar r^\perp_{ \psi}, \bar{r}_3)}, \quad \bar{r}_2 = { S(\bar{r}}_3) { \bar{r}}_1, \quad \\
\omega &= \begin{bmatrix} -r_2^\top u_v & r_1^\top u_v &u_w+\omega_r \end{bmatrix}^\top,
\end{aligned}
\end{equation}
where $u_x$, $u_y$, $u_z$ are given by \eqref{eq:double_integrator_control_law} for the corresponding error variables corresponding to $p_x$, $p_y$ and $p_z$, $u_v$ and $u_w$ are obtained by replacing $v$ by $r_3$ and $\bar{v}$ by $\bar{r}_3$ in~\eqref{eq:input},~\eqref{eq:input_u2} and~\eqref{eq:ff} and in~\eqref{eq:7},~\eqref{eq:18},~\eqref{eq:5},~\eqref{eq:2} respectively.

%

\section{Simulations}
\label{sec:simulations}
In this section, we present simulation results where we use the proposed saturated PD control law in the outer loop, and in the inner loop, we compare the proposed decoupled controller in $\mathcal{S}^2 \times \mathcal{S}^1$ with the discontinuous quaternion-based control proposed in~\cite{Mayhew2011,Fragopoulos2004} and the continuous controller in $\mathcal{S}\mathcal{O}(3)$ described in Remark~\ref{rem:1}. We consider the dynamics \eqref{eq:quadrotor_model} with $m = 1 [\text{kg}]$, and $g= 9.81[\text{m/s$^2$}]$. The outer-loop controller given by \eqref{eq:double_integrator_control_law} has the following parameters: $k_p = 1$ and $k_v = 2$.
For the inner loop we consider the $\mathcal{S}^2 \times \mathcal{S}^1$ controller proposed in Section~\ref{sec:3_3} and \ref{sec:3_4} with $k_1 = 2.5$ and $k_2 = 4$, the quaternion controller with $k = 5$, and the  $\mathcal{S}\mathcal{O}(3)$ controller \eqref{eq:continuous_SO3_law} with $K = 5I$, $k_1 = 0.9$, $k_2 = 1$, $k_3 = 1.1$. Please note that the aim of this simulation example to show the behavior of the different controllers and not the performance.

\par We compare these controllers in two scenarios that illustrate how the quaternion-based control and control in $\mathcal{SO}(3)$ described  in Remark~\ref{rem:1} can cause undesired behavior. In both scenarios, we let $p_x(0), p_y(0) = 0$ and $\bar p_x(t) = 1, \bar p_y(t) = 0, \bar{\psi}(t)=0, \Forall t > 0$.  Similarly to $\bar{\psi}$, we define $\psi$ to be the angle between the projection of $r_1=\begin{bmatrix}r_{1,1} & r_{1,2} & r_{1,3}
\end{bmatrix}$ on the  $xy$-plane and $e_1$, i.e., $\psi=\text{atan2}(r_{1,2},r_{1,1})^\top$. In the simulation $r_1$ is never perpendicular to the $xy$-plane  so that $\psi$ can always be computed. We denote $\psi$ by the yaw angle. The results with   $\psi(0) = 0$ are shown in Fig.~\ref{fig:simulation_no_yaw} and with $\psi(0) = \pi-0.01$ are shown in Fig.~\ref{fig:simulation_yaw}.

\begin{figure}[t]
   \centering
   \includegraphics[width=1\linewidth]{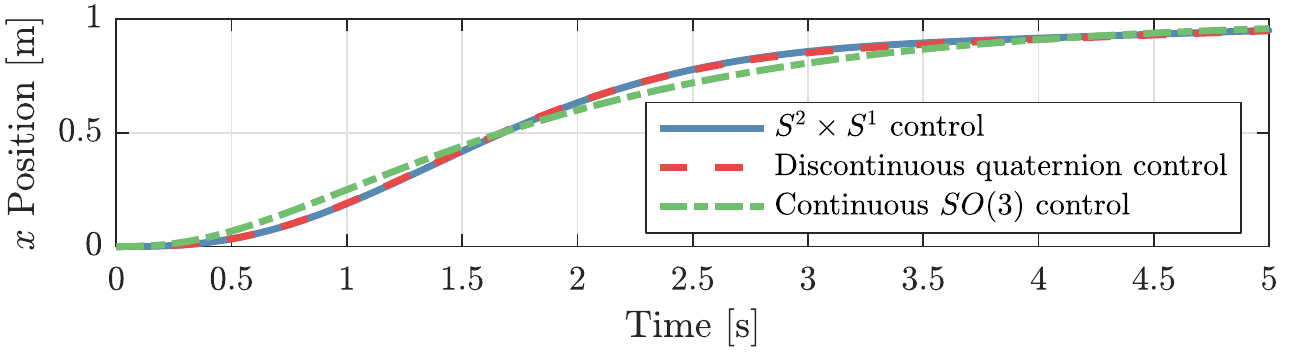} 
   \caption{Convergence of $x$ position from initial state $(x(0),y(0),\psi(0)) = (0,0,0)$ to reference setpoint $(1,0,0)$. For this initial state, position $y$  and yaw $\psi$ are zero for every $t$.}
   \label{fig:simulation_no_yaw}
\end{figure}
\begin{figure}[t] 
   \centering
   \includegraphics[width=1\linewidth]{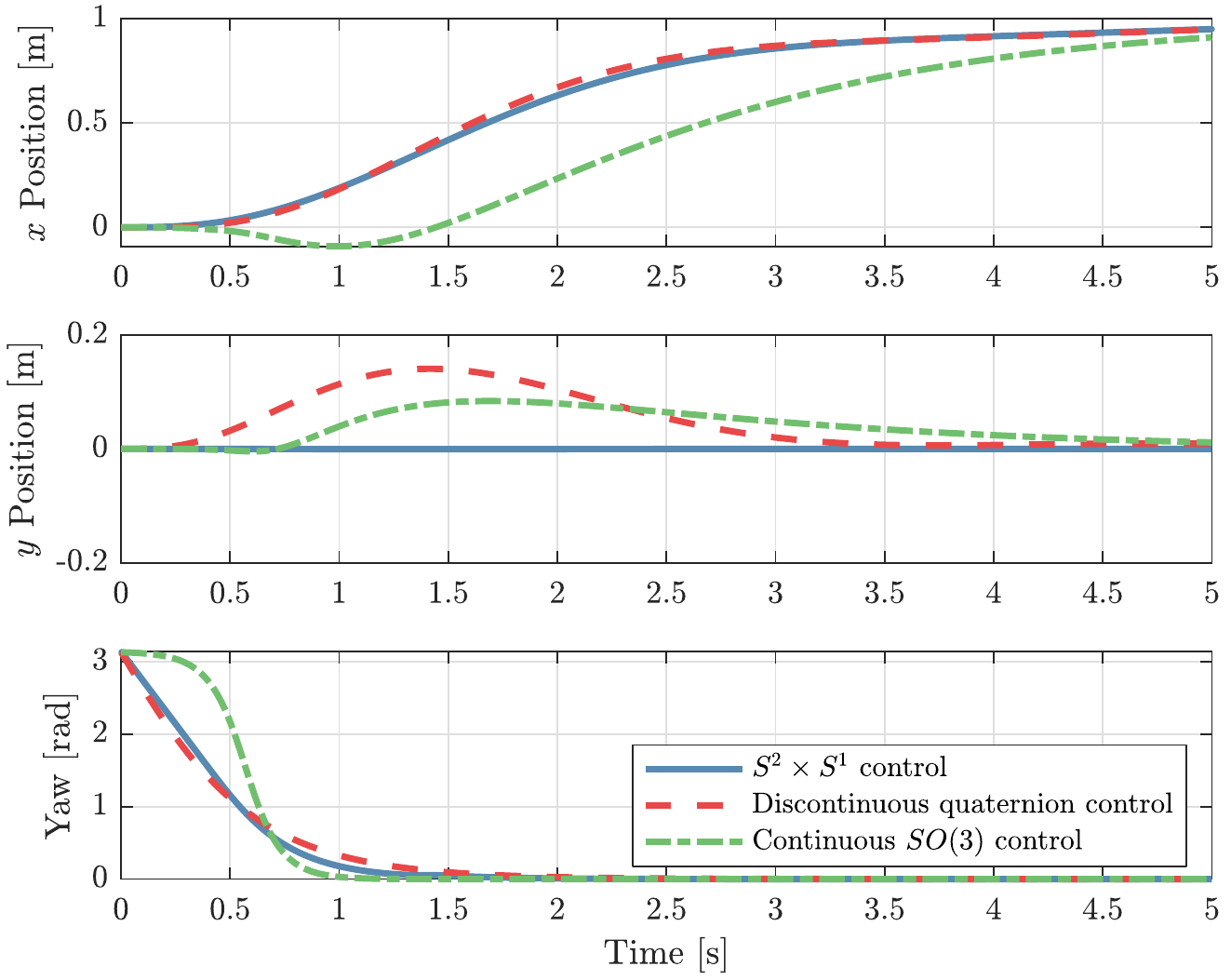} 
    \caption{Convergence of $x,y$ position and yaw angle from initial state $(x(0),y(0),\psi(0)) = (0,0,\pi-0.01)$ to reference setpoint $ (1,0,0)$. For this initial condition, $y$ is still zero for every $t$ for the proposed strategy, but the coupling between the control of the thrust vector and the yaw angle added in the other strategies leads to undesired position errors in $y$.}
   \label{fig:simulation_yaw}
\end{figure}

\par We can observe that the yaw error results in a position error in the $y$ direction when only movement in the $x$ direction is required for both the discontinuous quaternion control and continuous $\mathcal{SO}(3)$ control. In the proposed control law in $\mathcal{S}^2 \times \mathcal{S}^1$, the thrust vector is decoupled from the yaw and, therefore, a yaw error does not result in a position error.

\section{Conclusion}
\label{sec:conclusion}
We presented an attitude controller that achieves almost global asymptotic stability for the tracking error dynamics of a quadrotor, where the thrust vector is controlled in $S^2$ and the heading in $S^1$. By considering the decoupled control of the thrust vector and of the angle about this vector, we avoid undesired coupling of previous approaches. These advantages were illustrated in simulation. In future work, the stability of the overall closed-loop system, i.e., the quadrotor with the inner- and outer-loop controllers should be addressed.

\section{Acknowledgements}
\par The authors are thankful to A. A. J. (Erjen) Lefeber for several comments which have helped us improving the paper.


\bibliographystyle{IEEEtran}
\bibliography{root.bib}



\end{document}